%% file: iclr2027_conference.tex
\documentclass{article}

\usepackage{iclr2027_conference,times}

\input{math_commands.tex}

\usepackage[utf8]{inputenc}
\usepackage[T1]{fontenc}
\usepackage{hyperref}
\usepackage{url}
\usepackage{placeins}
\usepackage{booktabs}
\usepackage{amsfonts}
\usepackage{nicefrac}
\usepackage{microtype}
\usepackage{xcolor}
\usepackage{bm}
\usepackage{graphicx}
\usepackage{amsmath}
\usepackage{amssymb}
\usepackage{amsthm}
\usepackage{mathtools}
\usepackage{subfiles}
\usepackage{multirow}
\usepackage{wrapfig}
\usepackage[font=small,labelfont=bf]{caption}
\usepackage{float}
\usepackage{subcaption}

\newtheorem{proposition}{Proposition}

\title{
FAD-KV: Adaptive Dictionary KV Caching
for Efficient Reasoning Decoding
}

\author{%
\textbf{Yuxuan Zhu}\textsuperscript{1},
\textbf{David H. Yang}\textsuperscript{1},
\textbf{Parisa Shiravani}\textsuperscript{1},\\
\textbf{Mohammad Mohammadi Amiri}\textsuperscript{1},
\textbf{Tejaswini Pedapati}\textsuperscript{2},\\
\textbf{Subhajit Chaudhury}\textsuperscript{2},
\textbf{Keerthiram Murugesan}\textsuperscript{2},
\textbf{Pin-Yu Chen}\textsuperscript{2}
\\[0.5em]
\textsuperscript{1}Rensselaer Polytechnic Institute
\quad
\textsuperscript{2}IBM Research
\\[0.3em]
{\small\texttt{\{yangd13,zhuy27,shirap,mamiri\}@rpi.edu}}\\
{\small\texttt{tejaswinip@us.ibm.com}}\\
{\small\texttt{\{keerthiram.murugesan,subhajit,pin-yu.chen\}@ibm.com}}
}

\iclrfinalcopy

\begin{document}

\maketitle

\lhead{}

\begin{abstract}

Key-value (KV) caching is a core mechanism in large language models, where intermediate key and value representations are stored to accelerate autoregressive decoding.\, As generation length increases, KV cache memory and computation grow linearly, making compression critical for efficient long-output reasoning.\, Dictionary-based sparse KV compression is appealing because it represents each token's KV states as a sparse combination of learned basis elements (atoms) in a dictionary, preserving token-level structure without enforcing a global low-rank constraint.\, However, existing dictionary-based methods face two practical limitations during decoding: iterative sparse coding methods such as Orthogonal Matching Pursuit add decoding-time overhead, and static dictionaries may mismatch evolving KV distributions under test-time generation.\, While adaptive dictionary updates can reduce this mismatch, naive append-only strategies lead to dictionary growth during generation.\, We propose \textbf{FAD-KV}, a \textbf{F}ast \textbf{A}daptive \textbf{D}ictionary KV caching framework for efficient reasoning decoding.\, FAD-KV integrates (i) a screening-based sparse coding scheme, where a one-shot Top-\(n\) selection identifies a small candidate pool followed by lightweight coefficient refinement, and (ii) a fixed-size dictionary adaptation mechanism that reuses inactive atoms instead of expanding the dictionary.\, By controlling both sparse coding cost and adaptive dictionary growth, FAD-KV provides a memory-bounded and computation-efficient KV compression framework for long reasoning workloads.\, Experiments on reasoning and long-context benchmarks show that FAD-KV achieves a strong quality-efficiency trade-off while avoiding unbounded dictionary growth during decoding.\,
\end{abstract}

\subfile{Sections/Chapter1}


\subfile{Sections/Chapter3}

\subfile{Sections/Chapter4}

\subfile{Sections/Chapter5}

\subfile{Sections/Chapter6}

\bibliography{Reference}
\bibliographystyle{iclr2027_conference}

\appendix
\section{Additional Details of Sparse Coding Methods}
\label{app:sparse_coding}

This appendix provides additional details for the sparse coding methods used in FAD-KV.\, We first review full Orthogonal Matching Pursuit (OMP), then describe candidate-pruned greedy marginal coding, and finally explain the pool-ridge variant and its first-order redundancy-correction view.

\subsection{Related Works}
\label{app:relatedworks}

\paragraph{KV cache compression for efficient LLM inference.}
The KV cache is a major memory and bandwidth bottleneck in long-context decoding. Prior work reduces KV memory through token eviction or selection, streaming windows, quantization, and mixed-precision cache storage \citep{zhang2023h2o, xiao2023efficient, li2024snapkv, liu2024kivi, hooper2024kvquant, li2025kvtuner, dong2024qaq, he2024zipcache}. Recent reasoning-oriented methods further show that long generated chains shift the bottleneck from prompt prefill to decoding-time cache growth \citep{zhang2025lightthinker, yang2026zoomr, ramachandran2025thinkv}. FAD-KV follows this decoding-aware motivation, but compresses KV vectors through sparse dictionary representations rather than token removal or lower precision.

\paragraph{Dictionary-based and sparse coding-based KV representations.}
Sparse coding and dictionary learning represent data with a small number of atoms from overcomplete bases \citep{engan1999frame,aharon2006k,elad2010sparse}. Lexico applies this idea to activation-level KV cache compression, while CoSpaDi studies sparse dictionary representations for static LLM weights \citep{kim2024lexico, makhov2025cospadi}. Lexico reconstructs KV states using OMP and can compress without token eviction, but OMP requires iterative full-dictionary searches and append-only adaptation grows dictionary size. FAD-KV addresses these two decoding-time bottlenecks with candidate-screened sparse coding and fixed-size replacement.

\paragraph{Adaptive compression under test-time distribution shift.}
Compression schemes learned from calibration data can mismatch the hidden-state distribution induced by long reasoning. Existing adaptive cache methods mainly allocate memory across tokens, heads, layers, tasks, or decoding phases \citep{ge2023model,cai2024pyramidkv,feng2024ada,zhou2024dynamickv,wu2025scope}. FAD-KV is more closely related to representation-level adaptation. OjaKV~\citep{zhu2025ojakv} updates a low-rank subspace online, while Lexico~\citep{kim2024lexico} appends new dictionary atoms when reconstruction error is high. In contrast, FAD-KV keeps dictionary adaptation bounded by replacing inactive atoms under a fixed memory budget.

\subsection{Full OMP}
\label{app:full_omp}

Given an input vector \(\bm{x}\in\mathbb{R}^{d}\) and a dictionary
\[
\bm{D}=[\bm{d}_1,\dots,\bm{d}_N]\in\mathbb{R}^{d\times N},
\]
full OMP seeks an \(s\)-sparse approximation by greedily selecting atoms and refitting coefficients after each selection.\, Starting from
\[
\mathcal{S}_0=\emptyset,
\qquad
\bm{r}_0=\bm{x},
\]
OMP repeats the following steps for \(t=1,\dots,s\).\, First, it selects the atom that has the largest normalized correlation with the current residual:
\[
j_t
=
\arg\max_{j\notin \mathcal{S}_{t-1}}
\frac{(\bm{d}_j^\top\bm{r}_{t-1})^2}
{\|\bm{d}_j\|_2^2+\varepsilon},
\]
where \(\varepsilon>0\) is used for numerical stability.\, It then updates the selected atom index set
\[
\mathcal{S}_t=\mathcal{S}_{t-1}\cup\{j_t\},
\]
and refits the coefficients on the selected atoms:
\[
\bm{\beta}_t
=
\arg\min_{\bm{\beta}\in\mathbb{R}^{|\mathcal{S}_t|}}
\|\bm{x}-\bm{D}_{\mathcal{S}_t}\bm{\beta}\|_2^2
+
\lambda\|\bm{\beta}\|_2^2.
\]
When \(\lambda=0\), this is the standard least-squares refit; when \(\lambda>0\), it is a ridge-regularized refit.\, The residual is then updated as
\[
\bm{r}_t
=
\bm{x}-\bm{D}_{\mathcal{S}_t}\bm{\beta}_t.
\]
After \(s\) steps, OMP stores the selected indices \(\mathcal{S}_s\) and the corresponding coefficients \(\bm{\beta}_s\).

Full OMP is effective because every new atom is selected with respect to the current residual.\, This allows later choices to account for directions that have already been explained.\, However, this also makes OMP costly during autoregressive decoding: each greedy step scans all \(N\) atoms, and this procedure must be repeated for every newly generated key and value vector.

\subsection{Candidate-Pruned Greedy Marginal}
\label{app:candidate_pruned_gm}

Candidate-pruned greedy marginal coding keeps the residual-aware logic of OMP but reduces the search space before the greedy steps.\, It first computes one-shot correlation scores
\[
\bm{c}=|\bm{D}^\top\bm{x}|\in\mathbb{R}^{N},
\]
and forms a candidate pool
\[
\mathcal{C}
=
\mathrm{TopK}(\bm{c},n),
\qquad
n\ll N.
\]
The greedy search is then restricted to \(\mathcal{C}\).\, Starting from \(\mathcal{S}_0=\emptyset\) and \(\bm{r}_0=\bm{x}\), the method selects
\[
j_t
=
\arg\max_{j\in\mathcal{C}\setminus \mathcal{S}_{t-1}}
\frac{(\bm{d}_j^\top\bm{r}_{t-1})^2}
{\|\bm{d}_j\|_2^2+\varepsilon},
\]
updates \(\mathcal{S}_t=\mathcal{S}_{t-1}\cup\{j_t\}\), solves
\[
\bm{\beta}_t
=
\arg\min_{\bm{\beta}\in\mathbb{R}^{|\mathcal{S}_t|}}
\|\bm{x}-\bm{D}_{\mathcal{S}_t}\bm{\beta}\|_2^2
+
\lambda\|\bm{\beta}\|_2^2,
\]
and sets
\[
\bm{r}_t
=
\bm{x}-\bm{D}_{\mathcal{S}_t}\bm{\beta}_t.
\]
The procedure stops once \(|\mathcal{S}_t|=s\).

Thus, candidate-pruned greedy marginal differs from full OMP only in the search domain:
\[
\text{full OMP: } j\in\{1,\dots,N\},
\qquad
\text{candidate-pruned GM: } j\in\mathcal{C}.
\]
If the screened pool contains the atoms that full OMP would select, the two procedures can behave similarly.\, If the screening step misses useful atoms, candidate-pruned greedy marginal may deviate from full OMP.\, This is why the screening recall of \(\mathcal{C}\) is important.

The computational benefit comes from replacing repeated full-dictionary scans with scans over a much smaller pool.\, Full OMP has atom-search cost
\[
O(sNd),
\]
while candidate-pruned greedy marginal has screening cost \(O(Nd)\) and greedy-search cost
\[
O(snd),
\]
giving total atom-search cost
\[
O(Nd+snd).
\]
When \(n\ll N\) and \(s\) is not too large, this substantially reduces sparse coding overhead.

\subsection{Pool-Ridge and First-Order Redundancy Correction}
\label{app:pool_ridge_details}

Pool-ridge is a cheaper one-shot alternative to sequential greedy selection.\, It uses the same screened candidate pool
\[
\mathcal{C}
=
\mathrm{TopK}(|\bm{D}^\top\bm{x}|,n),
\]
but instead of adding atoms one by one, it solves a single ridge problem over all atoms in the pool:
\[
\bm{\beta}_{\mathcal{C}}
=
\arg\min_{\bm{u}\in\mathbb{R}^{|\mathcal{C}|}}
\|\bm{x}-\bm{D}_{\mathcal{C}}\bm{u}\|_2^2
+
\lambda\|\bm{u}\|_2^2.
\]
Equivalently,
\[
\bm{\beta}_{\mathcal{C}}
=
(\bm{D}_{\mathcal{C}}^\top\bm{D}_{\mathcal{C}}+\lambda \bm{I})^{-1}
\bm{D}_{\mathcal{C}}^\top\bm{x}.
\]
The final sparse code keeps the \(s\) largest entries of \(\bm{\beta}_{\mathcal{C}}\) in magnitude and maps them back to a full dictionary code \(\bm{y}\in\mathbb{R}^{N}\):
\[
y_j=
\begin{cases}
(\bm{\beta}_{\mathcal{C}})_j, & j\in \mathrm{TopK}(|\bm{\beta}_{\mathcal{C}}|,s),\\
0, & \text{otherwise}.
\end{cases}
\]
Here, \(\bm{\beta}_{\mathcal{C}}\) is indexed by atoms in \(\mathcal{C}\), so \(\mathrm{TopK}(|\bm{\beta}_{\mathcal{C}}|,s)\) returns the \(s\) selected atom indices from \(\mathcal{C}\).\, Optionally, one may refit the coefficients on this final support, but the main method already obtains a corrected coefficient estimate through the pooled ridge solve.

The main difference from raw top-\(s\) correlation is that pool-ridge accounts for interactions among atoms inside \(\mathcal{C}\).\, To see this, define
\[
\bm{G}_{\mathcal{C}}
=
\bm{D}_{\mathcal{C}}^\top\bm{D}_{\mathcal{C}},
\qquad
\bm{z}_{\mathcal{C}}
=
\bm{D}_{\mathcal{C}}^\top\bm{x}.
\]
Then
\[
\bm{\beta}_{\mathcal{C}}
=
(\bm{G}_{\mathcal{C}}+\lambda \bm{I})^{-1}\bm{z}_{\mathcal{C}}.
\]
Write the Gram matrix as
\[
\bm{G}_{\mathcal{C}}
=
\bm{\Lambda}+\bm{E},
\]
where \(\bm{\Lambda}\) is diagonal and \(\bm{E}\) contains off-diagonal atom correlations.\, Here,
\[
\Lambda_{jj}=\|\bm{d}_j\|_2^2,
\qquad
E_{ij}=\bm{d}_i^\top\bm{d}_j
\quad
(i\ne j).
\]
If \(\|\bm{E}\|\) is small relative to \(\bm{\Lambda}+\lambda \bm{I}\), a first-order expansion gives
\[
(\bm{G}_{\mathcal{C}}+\lambda \bm{I})^{-1}
=
(\bm{\Lambda}+\bm{E}+\lambda \bm{I})^{-1}
\approx
(\bm{\Lambda}+\lambda \bm{I})^{-1}
-
(\bm{\Lambda}+\lambda \bm{I})^{-1}
\bm{E}
(\bm{\Lambda}+\lambda \bm{I})^{-1}.
\]
Multiplying by \(\bm{z}_{\mathcal{C}}\), the coefficient for atom \(j\) satisfies
\[
\beta_j
\approx
\frac{z_j}{\|\bm{d}_j\|_2^2+\lambda}
-
\sum_{i\ne j}
\frac{
(\bm{d}_j^\top\bm{d}_i)z_i
}{
(\|\bm{d}_j\|_2^2+\lambda)(\|\bm{d}_i\|_2^2+\lambda)
}.
\]
The first term is a normalized relevance score: it is large when atom \(j\) is strongly aligned with \(\bm{x}\).\, The second term is a redundancy correction: if atom \(j\) is highly correlated with other atoms that also explain \(\bm{x}\), its effective coefficient is reduced.

This mirrors the residual correction used by greedy methods.\, In residual-based selection, the score for atom \(j\) depends on
\[
\bm{d}_j^\top\bm{r}_{\mathcal{S}}
=
\bm{d}_j^\top\bm{x}
-
\bm{d}_j^\top\bm{D}_{\mathcal{S}}\bm{\beta}_{\mathcal{S}}
=
z_j
-
\sum_{i\in \mathcal{S}}(\bm{d}_j^\top\bm{d}_i)\beta_i.
\]
Thus, both residual-based greedy selection and pool-ridge start from raw relevance and subtract information already explained by correlated atoms.\, The difference is that greedy marginal performs this correction sequentially through residual updates, while pool-ridge performs a one-shot correction through the inverse Gram matrix.

Pool-ridge is therefore not simply Top-\(s\) correlation.\, It is a one-shot coefficient-correction method that accounts for redundancy within the screened pool.\, It is usually cheaper than sequential greedy selection, but it can be less faithful to full OMP because it does not update the residual after each selected atom.

\section{Additional Theoretical Details}
\label{app:theory}

\subsection{Proof of Proposition~\ref{prop:attn_perturb}}
\label{app:proof_attn_perturb}

We prove the three inequalities in Proposition~\ref{prop:attn_perturb}.\, Let
\([t-1]:=\{1,\ldots,t-1\}\) be the attention index set.\, For each position
\(j\in[t-1]\), the attention-score perturbation satisfies
\[
|\hat z_{t,j}-z_{t,j}|
=
\left|
\frac{\langle \bm{q}_t,\hat{\bm{k}}_j\rangle}{\sqrt{d}}
-
\frac{\langle \bm{q}_t,\bm{k}_j\rangle}{\sqrt{d}}
\right|
=
\frac{|\langle \bm{q}_t,\hat{\bm{k}}_j-\bm{k}_j\rangle|}{\sqrt{d}}.
\]
By Cauchy--Schwarz,
\[
|\hat z_{t,j}-z_{t,j}|
\le
\frac{\|\bm{q}_t\|_2\|\hat{\bm{k}}_j-\bm{k}_j\|_2}{\sqrt{d}}
\le
\frac{Q E_k}{\sqrt{d}}.
\]
Taking the maximum over \(j\in[t-1]\) gives
\[
\|\hat{\bm{z}}_t-\bm{z}_t\|_\infty
\le
\frac{Q E_k}{\sqrt{d}}.
\]

Next, using the standard Lipschitz property of softmax from \(\ell_\infty\) to \(\ell_1\),
\[
\|\operatorname{softmax}(\hat{\bm{z}}_t)-\operatorname{softmax}(\bm{z}_t)\|_1
\le
2\|\hat{\bm{z}}_t-\bm{z}_t\|_\infty.
\]
Therefore,
\[
\|\hat{\bm{\alpha}}_t-\bm{\alpha}_t\|_1
\le
\frac{2Q E_k}{\sqrt{d}}.
\]

Finally, decompose the attention-output error as
\[
\hat{\bm{o}}_t-\bm{o}_t
=
\sum_{j\in[t-1]}\hat{\alpha}_{t,j}\hat{\bm{v}}_j
-
\sum_{j\in[t-1]}\alpha_{t,j}\bm{v}_j.
\]
Adding and subtracting
\(\sum_{j\in[t-1]}\alpha_{t,j}\hat{\bm{v}}_j\) gives
\[
\hat{\bm{o}}_t-\bm{o}_t
=
\sum_{j\in[t-1]}(\hat{\alpha}_{t,j}-\alpha_{t,j})\hat{\bm{v}}_j
+
\sum_{j\in[t-1]}\alpha_{t,j}(\hat{\bm{v}}_j-\bm{v}_j).
\]
By the triangle inequality,
\[
\|\hat{\bm{o}}_t-\bm{o}_t\|_2
\le
\sum_{j\in[t-1]}|\hat{\alpha}_{t,j}-\alpha_{t,j}|\|\hat{\bm{v}}_j\|_2
+
\sum_{j\in[t-1]}\alpha_{t,j}\|\hat{\bm{v}}_j-\bm{v}_j\|_2.
\]
Using \(\|\hat{\bm{v}}_j\|_2\le \hat V\),
\(\|\hat{\bm{v}}_j-\bm{v}_j\|_2\le E_v\), and
\(\sum_{j\in[t-1]}\alpha_{t,j}=1\), we obtain
\[
\|\hat{\bm{o}}_t-\bm{o}_t\|_2
\le
\hat V\|\hat{\bm{\alpha}}_t-\bm{\alpha}_t\|_1
+
E_v.
\]
Substituting the attention-weight bound yields
\[
\|\hat{\bm{o}}_t-\bm{o}_t\|_2
\le
\frac{2\hat V Q E_k}{\sqrt{d}}
+
E_v.
\]
This completes the proof.

\subsection{Bounded Memory under Fixed-Size Replacement}
\label{app:proof_bounded_memory}

\begin{proposition}[Bounded dictionary-plus-cache memory]
\label{prop:bounded_memory}
Consider a dictionary \(\bm{D}\in\mathbb{R}^{d\times N}\) during decoding, and suppose the compressed cache stores \(T\) sparse codes with sparsity budget \(s\).\, Each sparse code stores \(s\) atom indices and \(s\) coefficients.\, Under append-only adaptation, if \(A\) new atoms are appended, the total dictionary-plus-cache memory scales as
\[
O\!\left((N+A)d + Ts(b_{\mathrm{idx}}+b_{\mathrm{coef}})\right).
\]
Under fixed-size replacement, the dictionary always has exactly \(N\) atoms, so the total dictionary-plus-cache memory scales as
\[
O\!\left(Nd + Ts(b_{\mathrm{idx}}+b_{\mathrm{coef}})\right)
\]
for all decoding steps.
\end{proposition}

\begin{proof}
The compressed KV cache stores sparse representations rather than dense KV vectors.\,
For each cached vector, it stores \(s\) atom indices and \(s\) coefficients, so \(T\) cached vectors require
\[
O\!\left(Ts(b_{\mathrm{idx}}+b_{\mathrm{coef}})\right)
\]
memory, independent of whether the dictionary is append-only or fixed-size.

It remains to compare dictionary storage.\,
Under append-only adaptation, each triggered update adds one new atom to the dictionary.\,
If \(A\) updates are triggered during decoding, the dictionary contains \(N+A\) atoms, so its storage cost is
\[
O((N+A)d).
\]
Thus, the total dictionary-plus-cache memory is
\[
O\!\left((N+A)d + Ts(b_{\mathrm{idx}}+b_{\mathrm{coef}})\right).
\]

Under fixed-size replacement, each triggered update overwrites one existing atom,
\[
\bm{d}_{j^\star}\leftarrow \bm{a}_{\mathrm{new}},
\]
and the number of dictionary columns does not change.\,
Therefore, the dictionary always contains exactly \(N\) atoms, with storage cost \(O(Nd)\).\, 
Combining this with the sparse-cache storage gives
\[
O\!\left(Nd + Ts(b_{\mathrm{idx}}+b_{\mathrm{coef}})\right).
\]
This proves the claim.
\end{proof}

\subsection{Weak Submodularity of the Reconstruction Utility}
\label{app:weak_submodularity}

We provide more details for the weak-submodular view used in Section~\ref{sec:theory}.\, For a fixed vector \(\bm{x}\) and candidate atom index set \(\mathcal{C}\), define
\[
f(\mathcal{S})
=
\|\bm{x}\|_2^2
-
\min_{\bm{\beta}\in\mathbb{R}^{|\mathcal{S}|}}
\|\bm{x}-\bm{D}_{\mathcal{S}}\bm{\beta}\|_2^2,
\qquad
\mathcal{S}\subseteq\mathcal{C}.
\]
The function \(f(\mathcal{S})\) is the amount of squared norm of \(\bm{x}\) explained by the atoms indexed by \(\mathcal{S}\).\, Equivalently, when \(\bm{D}_{\mathcal{S}}^\top\bm{D}_{\mathcal{S}}\) is invertible, the least-squares solution gives
\[
f(\mathcal{S})
=
\bm{x}^\top
\bm{P}_{\mathcal{S}}
\bm{x},
\]
where
\[
\bm{P}_{\mathcal{S}}
=
\bm{D}_{\mathcal{S}}
(\bm{D}_{\mathcal{S}}^\top\bm{D}_{\mathcal{S}})^{-1}
\bm{D}_{\mathcal{S}}^\top
\]
is the projection matrix onto the span of \(\bm{D}_{\mathcal{S}}\).\, Thus, \(f(\mathcal{S})\) measures the energy of \(\bm{x}\) captured by the span of the selected atoms.

\paragraph{Monotonicity.}
For any \(\mathcal{A}\subseteq \mathcal{B}\subseteq\mathcal{C}\), the span of \(\bm{D}_{\mathcal{A}}\) is contained in the span of \(\bm{D}_{\mathcal{B}}\).\, Therefore,
\[
\min_{\bm{\beta}}
\|\bm{x}-\bm{D}_{\mathcal{B}}\bm{\beta}\|_2^2
\le
\min_{\bm{\beta}}
\|\bm{x}-\bm{D}_{\mathcal{A}}\bm{\beta}\|_2^2,
\]
which implies
\[
f(\mathcal{A})\le f(\mathcal{B}).
\]
Hence, the reconstruction utility is monotone.

\paragraph{Why exact submodularity may fail.}
A set function is submodular if it satisfies exact diminishing returns:
\[
f(\mathcal{A}\cup\{j\})-f(\mathcal{A})
\ge
f(\mathcal{B}\cup\{j\})-f(\mathcal{B}),
\qquad
\mathcal{A}\subseteq \mathcal{B},\; j\notin \mathcal{B}.
\]
For orthogonal dictionaries, reconstruction gains are additive across atoms, and this property is easier to satisfy.\, However, learned overcomplete dictionaries are generally correlated.\, The marginal gain of one atom can depend on which correlated atoms have already been selected.\, Because coefficients are refit jointly over \(\mathcal{S}\), adding an atom changes the optimal coefficients of previously selected atoms.\, This interaction prevents exact submodularity in general.

\paragraph{Weak submodularity.}
Although exact submodularity may fail, sparse regression objectives of this form are known to be weakly submodular under restricted conditioning assumptions~\citep{das2011submodular, elenberg2018restricted}.\, One common way to quantify this is through the submodularity ratio.\, For disjoint sets \(\mathcal{L},\mathcal{S}\subseteq\mathcal{C}\), define
\[
\gamma_{\mathcal{L},\mathcal{S}}
=
\frac{
\sum_{j\in \mathcal{S}}\left(f(\mathcal{L}\cup\{j\})-f(\mathcal{L})\right)
}{
f(\mathcal{L}\cup \mathcal{S})-f(\mathcal{L})
}.
\]
If \(\gamma_{\mathcal{L},\mathcal{S}}\) is bounded below by a positive constant \(\gamma>0\) for all relevant sets, then \(f\) is weakly submodular with ratio \(\gamma\).

For sparse regression, this ratio can be controlled by restricted eigenvalues of the dictionary.\, Suppose all columns in \(\bm{D}_{\mathcal{C}}\) are normalized and for all sets \(\mathcal{S}\) with \(|\mathcal{S}|\le 2s\),
\[
\mu_{2s}\|\bm{u}\|_2^2
\le
\|\bm{D}_{\mathcal{S}}\bm{u}\|_2^2
\le
L_{2s}\|\bm{u}\|_2^2.
\]
Then the submodularity ratio is lower bounded up to constants by the sparse condition number:
\[
\gamma
\gtrsim
\frac{\mu_{2s}}{L_{2s}}.
\]
When screened sub-dictionaries are well conditioned, \(\mu_{2s}/L_{2s}\) is not too small, and greedy selection can be expected to make steady progress.

\paragraph{Greedy approximation inside the screened pool.}
Let \(\mathcal{S}_s\) be the set selected by exact greedy maximization of marginal gain over \(\mathcal{C}\) under the constraint \(|\mathcal{S}_s|\le s\).\, For a monotone weakly submodular function with ratio \(\gamma\), the greedy solution satisfies
\[
f(\mathcal{S}_s)
\ge
\left(1-\exp(-\gamma)\right)
\max_{\substack{\mathcal{S}_s\subseteq\mathcal{C}\\ |\mathcal{S}_s|\le s}}
f(\mathcal{S}).
\]
This guarantee is relative to the candidate pool \(\mathcal{C}\), not the full dictionary.\, Therefore, FAD-KV depends on two separate conditions.\, First, the screening step should include high-utility atoms in \(\mathcal{C}\).\, Second, the atoms in \(\mathcal{C}\) should not be too ill conditioned.\, These two conditions motivate the empirical analysis of screening recall and reconstruction quality in Section~\ref{sec:experiments}.

\subsection{Connection between Greedy Marginal Coding and Reconstruction Utility}
\label{app:greedy_marginal_connection}

The exact marginal gain of adding atom \(j\) to a selected atom index set \(\mathcal{S}\) is
\[
\Delta(j\mid \mathcal{S})
=
f(\mathcal{S}\cup\{j\})-f(\mathcal{S}).
\]
Let
\[
\bm{r}_{\mathcal{S}}
=
\bm{x}-\bm{D}_{\mathcal{S}}\bm{\beta}_{\mathcal{S}}
\]
be the residual after the least-squares or ridge refit on \(\mathcal{S}\).\, If atom \(\bm{d}_j\) is normalized and we ignore the small effect of refitting the previous coefficients, the gain from adding \(j\) is approximately proportional to
\[
(\bm{d}_j^\top \bm{r}_{\mathcal{S}})^2.
\]
More generally, for non-normalized atoms, a scale-corrected score is
\[
\frac{(\bm{d}_j^\top \bm{r}_{\mathcal{S}})^2}{\|\bm{d}_j\|_2^2+\varepsilon}.
\]
This is the score used by the greedy marginal path in FAD-KV.\, Thus, the algorithm can be viewed as an efficient approximation to marginal-gain maximization of the reconstruction utility, restricted to the screened pool \(\mathcal{C}\).

\subsection{Connection between Pool-Ridge Coding and Redundancy Correction}
\label{app:pool_ridge_connection}

The pool-ridge variant can also be interpreted through the same reconstruction-utility view.\, Given a screened pool \(\mathcal{C}\), it solves
\[
\bm{\beta}_{\mathcal{C}}
=
(\bm{D}_{\mathcal{C}}^\top\bm{D}_{\mathcal{C}}+\lambda \bm{I})^{-1}
\bm{D}_{\mathcal{C}}^\top\bm{x}.
\]
Let
\[
\bm{G}_{\mathcal{C}}
=
\bm{D}_{\mathcal{C}}^\top\bm{D}_{\mathcal{C}},
\qquad
\bm{z}_{\mathcal{C}}
=
\bm{D}_{\mathcal{C}}^\top\bm{x}.
\]
Write
\[
\bm{G}_{\mathcal{C}}
=
\bm{\Lambda}+\bm{E},
\]
where \(\bm{\Lambda}\) is diagonal and \(\bm{E}\) contains the off-diagonal atom correlations.\, When \(\|\bm{E}\|\) is small relative to \(\bm{\Lambda}+\lambda \bm{I}\), a first-order expansion gives
\[
(\bm{G}_{\mathcal{C}}+\lambda \bm{I})^{-1}
\approx
(\bm{\Lambda}+\lambda \bm{I})^{-1}
-
(\bm{\Lambda}+\lambda \bm{I})^{-1}
\bm{E}
(\bm{\Lambda}+\lambda \bm{I})^{-1}.
\]
Therefore, the coefficient of atom \(j\) has the approximate form
\[
\beta_j
\approx
\frac{z_j}{\|\bm{d}_j\|_2^2+\lambda}
-
\sum_{i\ne j}
\frac{
(\bm{d}_j^\top\bm{d}_i)z_i
}{
(\|\bm{d}_j\|_2^2+\lambda)(\|\bm{d}_i\|_2^2+\lambda)
}.
\]
The first term is the raw relevance of atom \(j\) to \(\bm{x}\), while the second term subtracts redundancy explained by correlated atoms.\, This mirrors the residual correction used by greedy marginal coding, but performs it in one pooled solve rather than through sequential residual updates.

\section{Experimental Details}
\label{app:experimental_details}

This appendix provides additional details about the models, datasets, dictionary training procedure, decoding hyperparameters, and compute setup used in our experiments.

\subsection{Models and Datasets}
\label{app:models_datasets}

We evaluate FAD-KV on instruction-tuned decoder-only language models from
different model families, including Llama-3.2-3B-Instruct,
Llama-3.1-8B-Instruct, and Qwen3-4B-Instruct.

For reasoning evaluation, we use GSM8K, MATH500, GPQA-Diamond (GPQA-D),
and LongReason. GSM8K and MATH500 evaluate mathematical reasoning,
GPQA-D evaluates graduate-level scientific reasoning, and LongReason
targets longer-form reasoning under extended generation.

For long-context evaluation, we use RULER and LongBench. RULER evaluates
long-range retrieval at 8K and 16K context lengths, while LongBench covers
single-document QA, multi-document QA, summarization, and code completion.
Together, these benchmarks test whether compressed KV representations preserve
reasoning quality and long-context retrieval and generation behavior.

All methods use the same prompts, decoding settings, answer extraction rules,
and scoring scripts within each benchmark. For reasoning tasks, we extract the
final answer and apply the corresponding benchmark-specific evaluation rule.
For RULER and LongBench, we follow their standard task-specific metrics.

\subsection{Dictionary Training Details}
\label{app:dictionary_training}

For each model, we train separate dictionaries for key and value vectors using KV activations collected from WikiText-103~\citep{merity2016wikitext} as the calibration set.\, Dictionaries are layer-specific, and keys and values are treated independently unless otherwise stated.\, Therefore, each layer has one key dictionary and one value dictionary.

Given calibration vectors $\{\bm{x}_i\}$ from a fixed layer and KV type, dictionary training follows the sparse reconstruction objective
\[
\min_{\bm{D},\{\bm{y}_i\}}
\sum_i \|\bm{x}_i-\bm{D}\bm{y}_i\|_2^2
\qquad
\mathrm{s.t.}\qquad
\|\bm{y}_i\|_0\le s,
\]
where $\bm{D}\in\mathbb{R}^{d\times N}$ is the learned dictionary, $N$ is the number of atoms, $d$ is the head dimension or vector dimension being compressed, and $s$ is the sparsity budget. During inference, the learned dictionaries are fixed initially and then updated only by the adaptive mechanisms described in the main text.

For FAD-KV, the offline dictionary is used as the base dictionary. The fixed-size adaptive update does not increase the number of atoms. When adaptation is triggered, one currently inactive atom is replaced by a context-relevant atom constructed from the current vector or residual. Since the replaced atom is not referenced by any currently cached sparse code, previously encoded cache entries remain valid and are not re-encoded after replacement.

\subsection{Decoding Settings and Hyperparameters}
\label{app:decoding_hyperparams}

All methods use the same decoding configuration within each benchmark and model. During prefilling, prompt KV vectors are encoded and stored in the compressed cache. During decoding, each newly generated key and value vector is encoded before being inserted into the cache. Full KV stores dense keys and values and serves as the reference.

For dictionary-based methods, the main hyperparameters are the dictionary size \(N\), sparsity budget \(s\), candidate-pool size \(n\), ridge parameter \(\lambda\), and adaptation threshold \(\delta\).\, Greedy Marginal uses one-shot screening to form a candidate pool \(\mathcal{C}\) of size \(n\), then performs residual-aware greedy selection inside \(\mathcal{C}\) until \(s\) atoms are selected.\, Pool-Ridge uses the same screened pool but solves one ridge problem over \(\mathcal{C}\) and keeps the Top-\(s\) coefficients by magnitude.

The reconstruction error used for triggering adaptation is
\[
e(\bm{x};\bm{D})=\frac{\|\bm{x}-\bm{D}\bm{y}\|_2}{\|\bm{x}\|_2}.
\]
If \(e(\bm{x};\bm{D})>\delta\), an adaptive update is triggered. In append-only adaptive baselines, a new atom is appended to the dictionary. In FAD-KV, one currently inactive atom is replaced, keeping the dictionary size fixed. Usage counts are updated whenever the corresponding entry of a stored sparse code is nonzero, and replacement is restricted to atoms with zero current cache usage.

Unless otherwise specified, all dictionary-based methods use the same base dictionary size, sparsity budget, and adaptation threshold within the same experiment. This ensures that comparisons isolate the effects of sparse coding path and adaptive update strategy rather than changes in the base memory budget.

Detailed parameter settings are shown in Table~\ref{tab:hyperparams_appendix}.

\begin{table}[!htbp]
\centering
\caption{Default hyperparameters used in dictionary-based experiments.}
\label{tab:hyperparams_appendix}
\scriptsize
\setlength{\tabcolsep}{4pt}
\begin{tabular}{lcc}
\toprule
\textbf{Hyperparameter} & \textbf{Symbol} & \textbf{Value} \\
\midrule
Dictionary size & \(N\) & 4096 (3B and 4B models); 8192 (8B model) \\
Sparsity budget & \(s\) & 8 (3B and 4B models); 16 (8B model)\\
Candidate-pool size & \(n\) & 256 (3B and 4B models); 512 (8B model) \\
Ridge parameter & \(\lambda\) & \(10^{-4}+0.20\,\overline{\mathrm{diag}(\bm{G})}\) \\
Adaptation threshold & \(\delta\) & 0.50 (key) / 0.70 (value) \\
Maximum new tokens &  & 1024 (MATH500, GSM8K, LongBench)\\
Decoding temperature &  & 0 (greedy) \\
Top-\(p\) / top-\(k\) sampling & & disabled \\
\bottomrule
\end{tabular}
\end{table}

Here, \(\bm{G}=\bm{D}_{\mathcal{C}}^\top\bm{D}_{\mathcal{C}}\) is the pool Gram matrix used in the ridge solve, and \(\overline{\mathrm{diag}(\bm{G})}\) denotes the mean of its diagonal entries.

\subsection{Compute Resources}
\label{app:compute_resources}

All experiments were run on NVIDIA H100 96GB GPUs.\,
Latency and memory measurements were collected with batch size 1 under the same inference backend and decoding configuration for all compared methods within each model.\,
Dictionary-based methods store dictionaries on GPU and store sparse atom indices and coefficients in the compressed cache.\,
Full KV stores dense key and value tensors and is used as the dense-cache reference.

\end{document}

%% file: math_commands.tex
\usepackage{amsmath,amsfonts,bm}

\def\eqref#1{equation~\ref{#1}}

\def\1{\bm{1}}

\DeclareMathAlphabet{\mathsfit}{\encodingdefault}{\sfdefault}{m}{sl}
\SetMathAlphabet{\mathsfit}{bold}{\encodingdefault}{\sfdefault}{bx}{n}

\newcommand{\softmax}{\mathrm{softmax}}



%% file: Sections/Chapter1.tex
\section{Introduction}

Large language models (LLMs) increasingly rely on long-form reasoning during inference to solve tasks such as mathematical problem solving, code generation, and multi-step decision making~\citep{nam2024using,setlur2024rl,kumichev2024medsyn}.\, In these settings, autoregressive decoding can produce long outputs and maintain a large effective context, making the key-value (KV) cache a major bottleneck.\, Since standard KV caching stores one key and one value vector for every processed token at every layer, its memory footprint and memory traffic grow linearly with sequence length, which can substantially reduce decoding efficiency in long-context and reasoning workloads \citep{xiao2023efficient, zhang2023h2o, liu2024kivi, tang2024quest, kim2024lexico}.

A broad line of work studies KV cache compression through quantization \citep{hooper2024kvquant, liu2024kivi}, token eviction or selection \citep{xiao2023efficient, zhang2023h2o, li2024snapkv, tang2024quest, zhu2025sentencekv}, and low-rank or subspace approximation \citep{saxena2024eigen, lin2024matryoshkakv}.\, Dictionary-based sparse compression is a complementary direction: it represents each KV vector with a small number of atoms from a learned dictionary, preserving token-level information without directly evicting tokens or forcing all vectors into a shared low-rank subspace \citep{kim2024lexico}.\, However, applying dictionary-based KV compression to long reasoning decoding raises two practical challenges.\, First, sparse coding with iterative solvers such as Orthogonal Matching Pursuit (OMP) can introduce noticeable per-token latency.\, Second, because the dictionary is learned offline, generated KV states may deviate from the calibration distribution during long reasoning.\, Adaptive dictionary updates can reduce this mismatch, but a naive append-only strategy causes the dictionary to grow during generation, weakening the memory savings of compression.

To make this concrete, dictionary-based KV compression represents each key or value vector \(\bm{x}\) as a sparse linear combination of atoms from an offline-learned dictionary \(\bm{D}=[\bm{d}_1,\dots,\bm{d}_N]\), i.e., \(\bm{x} \approx \bm{D}\bm{y}\).\, At inference time, a greedy sparse coding solver such as OMP selects a small support of atoms and estimates the corresponding sparse coefficient vector \(\bm{y}\) for each incoming KV vector where \(\|\bm{y}\|_0\le s\), where $s$ is the sparsity budget, and the compressed cache stores these sparse indices and coefficients instead of the dense vector.\, Full-dictionary OMP is expensive during autoregressive decoding because each greedy step scans all \(N\) atoms using the current residual, refits the coefficients on the selected support, and repeats this process until the sparsity budget is reached.\, Since this encoding is repeated for every newly generated key and value vector, full OMP can add substantial decoding-time overhead.\, Meanwhile, because \(\bm{D}\) is learned offline from calibration data, its reconstruction quality can degrade when generated KV states drift during long reasoning.\, A naive way to reduce this mismatch is to append new atoms from poorly reconstructed KV vectors, but such append-only adaptation makes the dictionary size grow with generation length.\, These two observations motivate our goal: reduce the sparse-coding search cost while keeping dictionary adaptation memory-bounded.\,
\begin{figure}[!t]
    \centering
    \includegraphics[width=\linewidth]{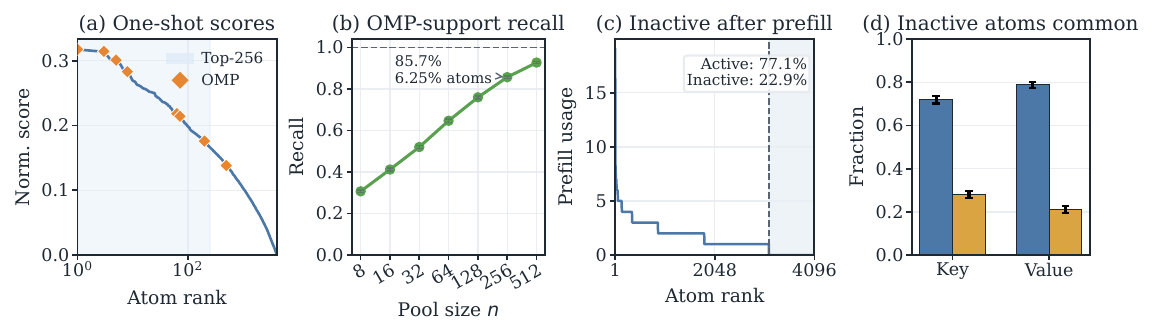}
    \caption{
    Motivation for FAD-KV.\ Diagnostics are computed using Llama-3.2-3B on MATH-500 with \(4096\)-atom key/value dictionaries and sparsity \(s=8\).\ (a)--(b) Decoding-time KV vectors have head-heavy dictionary correlations; we visualize value-vector scores as an example, and show that a small Top-\(n\) screened pool recalls most full-OMP atoms.\ (c)--(d) many atoms remain inactive after prefill, and inactive capacity is common across prompts.\ These observations motivate screening-based sparse coding and fixed-size dictionary adaptation.
    }
    \label{fig:intro_motivation}
    \vspace{-.6cm}
\end{figure}

Our motivation is that both challenges can be addressed by exploiting two forms of unused structure in dictionary-based KV compression: sparse correlation structure during atom selection and inactive dictionary capacity during adaptation.\, We illustrate these structures with a diagnostic from our main experimental setting, using Llama-3.2-3B-Instruct~\citep{grattafiori2024llama3herdmodels} model on MATH-500 dataset~\citep{hendrycks2021measuring, lightman2023let} with a dictionary of \(4096\)-atom and sparsity budget \(s=8\).\, For a KV vector \(\bm{x}\) and unit-normalized dictionary atoms \(\bm{d}_j\), we define the one-shot normalized correlation score as \(c_j = |\bm{d}_j^\top \bm{x}| / \|\bm{x}\|_2\), and form a candidate pool by taking the Top-\(n\) atoms under this score.\, For sparse coding, decoding-time KV vectors often have head-heavy dictionary correlations: as shown in Figure~\ref{fig:intro_motivation}(a), only a small fraction of atoms has large normalized correlation \(c_j\) with the current KV vector \(\bm{x}\), visualized here using a representative value vector.\, More importantly, this high-score region already contains most atoms that full OMP would select.\, Figure~\ref{fig:intro_motivation}(b) shows that a Top-\(256\) pool recalls \(85.7\%\) of the full-OMP support while searching only \(6.25\%\) of the dictionary.\, This suggests that expensive full-dictionary OMP search is often unnecessary: one-shot screening can first prune the dictionary to a small high-recall candidate pool, after which lightweight atom selection and coefficient correction can operate only within this pool.\,


For dictionary adaptation, we observe that the offline dictionary also contains unused capacity at test time.\, The dictionary is trained on calibration activations, but long reasoning trajectories may shift away from this offline distribution, causing some KV vectors to be poorly reconstructed.\, A natural response is to append new atoms online, but repeated appends make the memory budget input-dependent and costly for longer outputs.\, In contrast, Figure~\ref{fig:intro_motivation}(c) shows that \(22.9\%\) of dictionary atoms remain inactive after prefill for a representative prompt, and Figure~\ref{fig:intro_motivation}(d) shows that the inactive capacity is common across prompts for both key and value dictionaries.\, This motivates fixed-size adaptation, where instead of growing the dictionary, we replace currently inactive atoms with context-relevant adaptive atoms.

We propose \textbf{FAD-KV} (\textbf{F}ast \textbf{A}daptive \textbf{D}ictionary \textbf{KV} caching), a dictionary-based KV compression framework for efficient reasoning decoding.\, FAD-KV combines two inference-time mechanisms.\, First, it replaces full OMP over the entire dictionary with screening-based sparse approximation: a one-shot Top-\(n\) step identifies a high-recall candidate pool, and a lightweight correction step selects and refits atoms inside this pool.\, Second, it replaces append-only dictionary growth with a fixed-size update rule that swaps inactive atoms with context-relevant adaptive atoms.\, As a result, FAD-KV controls both sparse coding overhead and dictionary memory growth under a strict memory budget.\ For related works, refer to Appendix \ref{app:relatedworks}.

Our contributions are threefold.\,
First, we identify two practical bottlenecks in dictionary-based KV compression for reasoning-heavy decoding: iterative sparse coding latency and unbounded adaptive dictionary growth under test-time distribution shift.\,
Second, we propose FAD-KV, which combines screening-based sparse coding with fixed-size online dictionary adaptation.\,
Third, we show that FAD-KV reduces sparse coding overhead and avoids adaptive dictionary growth while providing a strong quality-efficiency trade-off under a fixed memory budget.

%% file: Sections/Chapter3.tex
\vspace{-.2cm}
\section{Dictionary-Based KV Compression}
\label{sec:background}

\subsection{Attention and Sparse KV Representation}

For layer $\ell$, head $h$, and decoding step $t$, let
$\bm{q}_t^{(\ell,h)}\in\mathbb{R}^{1\times d}$ denote the current query vector, and let
$\bm{K}^{(\ell,h)},\bm{V}^{(\ell,h)}\in\mathbb{R}^{(t-1)\times d}$
denote the cached keys and values from previous positions.\,
Standard causal attention computes
\setlength{\abovedisplayskip}{4pt}
\setlength{\belowdisplayskip}{4pt}
\[
\bm{z}_t^{(\ell,h)}
=
\bm{q}_t^{(\ell,h)}(\bm{K}^{(\ell,h)})^\top / \sqrt{d},
\qquad
\bm{\alpha}_t^{(\ell,h)}
=
\mathrm{softmax}(\bm{z}_t^{(\ell,h)}),
\qquad
\bm{o}_t^{(\ell,h)}
=
\bm{\alpha}_t^{(\ell,h)}\bm{V}^{(\ell,h)}.
\]
Since all past keys and values are retained for future decoding steps, KV cache memory grows linearly with sequence length.

Dictionary-based KV compression instead represents each key or value by a sparse linear combination of learned atoms.\,
Let
$\bm{D}_k^{(\ell,h)}\in\mathbb{R}^{d\times N_k}$ and
$\bm{D}_v^{(\ell,h)}\in\mathbb{R}^{d\times N_v}$
be the key and value dictionaries for layer $\ell$ and head $h$.
Each key and value is reconstructed as
\[
\hat{\bm{k}}_t^{(\ell,h)}
=
\bm{D}_k^{(\ell,h)}\bm{y}_{k,t}^{(\ell,h)},
\qquad
\hat{\bm{v}}_t^{(\ell,h)}
=
\bm{D}_v^{(\ell,h)}\bm{y}_{v,t}^{(\ell,h)},
\]
where
$\bm{y}_{k,t}^{(\ell,h)}\in\mathbb{R}^{N_k}$ and
$\bm{y}_{v,t}^{(\ell,h)}\in\mathbb{R}^{N_v}$
are sparse coefficient vectors.
The GPU-resident compressed cache stores only the sparse indices and nonzero coefficients of
$\bm{y}_{k,t}^{(\ell,h)}$ and $\bm{y}_{v,t}^{(\ell,h)}$,
while the dictionaries
$\bm{D}_k^{(\ell,h)}$ and $\bm{D}_v^{(\ell,h)}$
are stored once and shared across cached tokens.

Given reconstructed cached matrices
$\hat{\bm{K}}^{(\ell,h)},\hat{\bm{V}}^{(\ell,h)}
\in\mathbb{R}^{(t-1)\times d}$,
compressed attention becomes
\[
\hat{\bm{z}}_t^{(\ell,h)}
=
\bm{q}_t^{(\ell,h)}(\hat{\bm{K}}^{(\ell,h)})^\top / \sqrt{d},
\qquad
\hat{\bm{\alpha}}_t^{(\ell,h)}
=
\mathrm{softmax}(\hat{\bm{z}}_t^{(\ell,h)}),
\qquad
\hat{\bm{o}}_t^{(\ell,h)}
=
\hat{\bm{\alpha}}_t^{(\ell,h)}\hat{\bm{V}}^{(\ell,h)}.
\]

\subsection{Offline Training and Inference-Time Encoding}

The dictionaries are trained offline using KV activations from a calibration set,
following sparse dictionary learning \citep{aharon2006k, kim2024lexico}.
Given training key or value vectors $\{\bm{x}_i\}$, offline training jointly
optimizes the dictionary and sparse codes, whereas inference fixes the learned
dictionary and solves only for the sparse code of each new vector:
\begin{equation}
\begin{aligned}
\text{offline:}\quad
&\min\nolimits_{\bm{D},\{\bm{y}_i\}}
\sum\nolimits_i \|\bm{x}_i-\bm{D}\bm{y}_i\|_2^2
\quad
\mathrm{s.t.}\quad
\|\bm{y}_i\|_0 \le s, \\
\text{inference:}\quad
&\min\nolimits_{\bm{y}\in\mathbb{R}^{N}}
\|\bm{x}-\bm{D}\bm{y}\|_2^2
\quad
\mathrm{s.t.}\quad
\|\bm{y}\|_0\le s.
\end{aligned}
\label{eq:sparse_inference}
\end{equation}
Here, $\bm{x}\in\mathbb{R}^{d}$ is the key or value vector to encode,
$s$ is the sparsity budget, and
$\bm{D}=[\bm{d}_1,\dots,\bm{d}_N]\in\mathbb{R}^{d\times N}$ is the
corresponding dictionary.
OMP \citep{pati1993orthogonal} is a common inference-time solver, greedily
selecting atoms from the current residual and refitting their coefficients.
This repeated search is costly during autoregressive decoding, and the
offline-trained dictionary may also mismatch test-time KV states.
\subsection{Offline Dictionaries under Test-Time Shift}

Offline dictionaries may not fully cover the KV states produced during long reasoning trajectories.\, To detect such mismatch, one can monitor the relative reconstruction error
\(e(\bm{x};\bm{D})=\|\bm{x}-\bm{D}\bm{y}\|_2/\|\bm{x}\|_2\),
where \(\bm{y}\) is the sparse code produced by inference-time encoding.\, When this error is large, the current dictionary may not provide a good representation of the new vector \(\bm{x}\).\, A direct adaptive strategy is to append a new atom derived from the current vector or residual.\, However, repeated appends make the dictionary size grow during generation, so the memory budget becomes input-dependent.\, This creates a tension between test-time adaptivity and bounded memory, motivating the fixed-size adaptive update rule introduced in Section~\ref{sec:method}.

%% file: Sections/Chapter4.tex
\section{Methodology}
\label{sec:method}

\begin{figure}[t]
    \centering
    \includegraphics[width=0.85\linewidth]{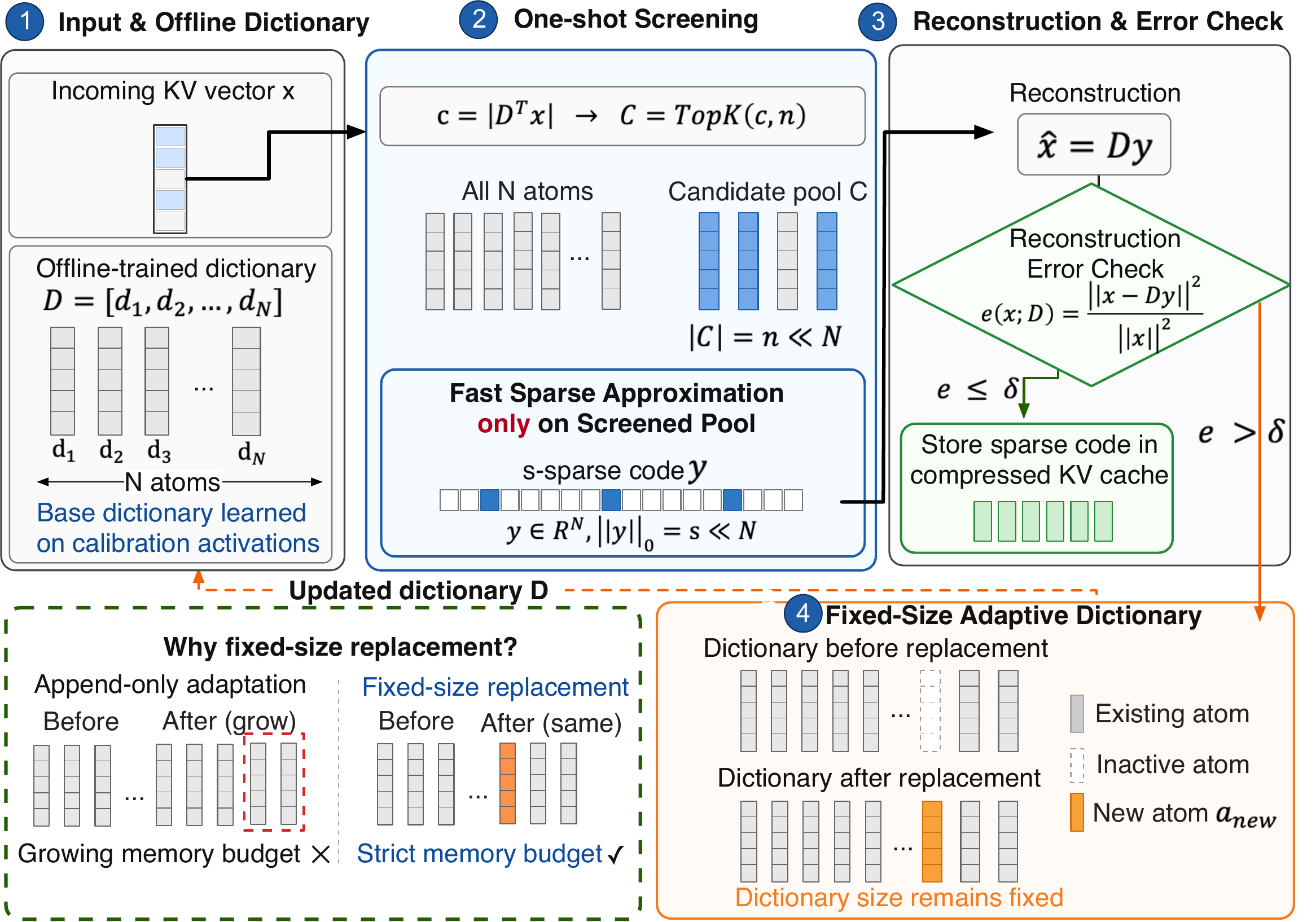}
    \caption{
    Overview of FAD-KV.\, Each incoming KV vector is first encoded through screening-based sparse coding.\, If reconstruction error is high, FAD-KV replaces an inactive atom with a context-relevant atom instead of appending a new one, keeping the dictionary size fixed throughout decoding.
    }
    \label{fig:fadkv_overview}
\end{figure}

We propose \textbf{FAD-KV}, a fixed-size adaptive dictionary KV caching framework for efficient reasoning decoding.\ Starting from offline-trained key and value dictionaries, FAD-KV modifies only the inference-time compression procedure.\ For each incoming key or value vector $\bm{x}\in\mathbb{R}^{d}$, we encode it over the current dictionary
\(
\bm{D}=[\bm{d}_1,\dots,\bm{d}_N]\in\mathbb{R}^{d\times N},
\)
where $N$ is the number of atoms in the key or value dictionary.\ The compressed cache stores the sparse representation (the selected atom indices and their nonzero coefficients) of the code \(\bm{y}\) for both key and value, rather than the dense KV vector \(\bm{x}\).\, The method targets two inference-time bottlenecks of dictionary-based KV compression: full OMP finding the sparse code \(\bm{y}\) is too expensive for token-wise decoding, and naive append-only dictionary adaptation causes input-dependent dictionary growth.\, FAD-KV addresses these issues with candidate-screened sparse coding and fixed-size dictionary replacement, as shown in Figure~\ref{fig:fadkv_overview}.

\subsection{Candidate-Screened Sparse Coding}

A standard sparse approximation solver to find \(\bm{y}\) in~\eqref{eq:sparse_inference} is OMP, which greedily selects atoms using the current residual and refits coefficients after each step.\, Although effective, OMP repeatedly scans the full dictionary and is costly when applied to every newly generated key and value vector.\, We assume dictionary atoms are $\ell_2$-normalized after offline training and after each adaptive replacement.\, 
For an input key or value vector \(\bm{x}\), FAD-KV computes
\setlength{\abovedisplayskip}{4pt}
\setlength{\belowdisplayskip}{4pt}
\begin{equation}
\bm{c}=|\bm{D}^{\top}\bm{x}|,\qquad
c_{\max}(\bm{x})=\|\bm{c}\|_{\infty}/\|\bm{x}\|_2,\qquad
\mathcal{C}=\mathrm{TopK}(\bm{c},n).
\label{eq:screening}
\end{equation}
Here, \(\bm{c}\in\mathbb{R}^{N}\) contains the absolute atom-input correlation scores, \(c_{\max}(\bm{x})\) is a lightweight confidence score, \(\mathcal{C}\) is the index set of the Top-\(n\) atoms with the largest scores \(|\bm{d}_j^\top\bm{x}|\), and $n\ll N$ is the candidate-pool size.\, If \(c_{\max}(\bm{x})\ge\tau_c\), FAD-KV restricts sparse coding to atoms in \(\mathcal{C}\); otherwise, it uses the conservative full-dictionary OMP solver.\, The score $c_{\max}(\bm{x})$ measures whether the current vector has at least one strong match in the offline dictionary.\,
When $c_{\max}(\bm{x})$ is relatively small, the Top-$n$ pool is less reliable because even the best atom has weak alignment with $\bm{x}$.\,
We therefore use $\tau_c$ as a conservative fallback threshold: if $c_{\max}(\bm{x})<\tau_c$, we disable screening and use full OMP.\,
In our experiments, $\tau_c$ is fixed on a small calibration split and then kept unchanged across evaluated tasks.\, In the screened case, i.e., when \(c_{\max}(\bm{x})\ge\tau_c\), FAD-KV solves for an \(s\)-sparse code \(\bm{y}\in\mathbb{R}^{N}\) whose nonzero entries are restricted to \(\mathcal{C}\):
\(
\|\bm{y}\|_0 \le s,
y_j=0\ \text{for all } j\notin\mathcal{C},
\hat{\bm{x}}=\bm{D}\bm{y}
\).\, The compressed cache stores only the nonzero indices and coefficients of \(\bm{y}\), while the code itself remains an \(N\)-dimensional sparse vector over the original dictionary.\, 

After screening, the remaining question is how to compute this restricted sparse code using only atoms indexed by \(\mathcal{C}\).\, We next describe two choices: greedy marginal coding, which selects nonzero entries of \(\bm{y}\) sequentially using residual scores, and pool-ridge coding, which estimates coefficients over the whole pool and then keeps the $s$ largest-magnitude entries.\, Additional details on full OMP and candidate-pruned greedy marginal coding are provided in Appendix~\ref{app:sparse_coding}.

\subsubsection{Greedy Marginal Coding}

After one-shot screening in~\eqref{eq:screening}, the goal is to compute an \(s\)-sparse code \(\bm{y}\in\mathbb{R}^{N}\) whose nonzero entries are restricted to the candidate index set \(\mathcal{C}\).\, Our default fast path is candidate-pruned greedy marginal coding.\, It keeps the residual-aware selection logic of OMP but restricts the search to $\mathcal{C}$.\, Starting from $\mathcal{S}_0=\emptyset$ and $\bm{r}_0=\bm{x}$, at step $t$ we select
\begin{equation}
j_t
=
\arg\max_{j\in\mathcal{C}\setminus \mathcal{S}_{t-1}}
\frac{(\bm{d}_j^\top \bm{r}_{t-1})^2}
{\|\bm{d}_j\|_2^2+\varepsilon},
\qquad
\mathcal{S}_t=\mathcal{S}_{t-1}\cup\{j_t\}.
\label{eq:greedy_marginal_selection}
\end{equation}
Here, \(\mathcal{S}_t\) is the set of selected atom indices after \(t\) steps.\, We then refit coefficients on the selected atoms and update the residual:
\[
\bm{\beta}_t
=
\arg\min_{\bm{\beta}\in\mathbb{R}^{|\mathcal{S}_t|}}
\|\bm{x}-\bm{D}_{\mathcal{S}_t}\bm{\beta}\|_2^2+\lambda\|\bm{\beta}\|_2^2,
\qquad
\bm{r}_t=\bm{x}-\bm{D}_{\mathcal{S}_t}\bm{\beta}_t.
\]
Here, \(\varepsilon>0\) is a small constant for numerical stability and
\(\lambda\ge0\) is a regularization parameter~\citep{hoerl1970ridge}.\,
After $s$ selections, let $\mathcal{S}_s=(j_1,\dots,j_s)$ and $\bm{\beta}_s\in\mathbb{R}^{s}$ be the final support and coefficients.\, We store $\bm{y}\in\mathbb{R}^{N}$ by
\[
y_j= (\bm{\beta}_s)_p, \; \text{if } j=j_p \text{ for some } p\in\{1,\dots,s\}, \qquad y_j=0, \; \text{otherwise}.
\]
Thus, the nonzero indices of \(\bm{y}\) are selected from \(\mathcal{C}\), and the compressed cache stores these nonzero indices and coefficients.\, Intuitively, after some atoms have already been selected, the residual \(\bm{r}_{t-1}\) is the part of \(\bm{x}\) that is not yet explained by the current approximation.\, Greedy marginal coding therefore selects, from the remaining indices in \(\mathcal{C}\), the atom whose direction is most aligned with this residual.\, Compared with full OMP, the main change is the search domain: FAD-KV searches over \(\mathcal{C}\) instead of all \(N\) atoms.\, Thus, if the screening step captures the useful atoms, greedy marginal coding can approximate OMP while avoiding repeated full-dictionary scans.

\subsubsection{Pool-Ridge Coding}

We also provide a cheaper one-shot variant for identifying sparse code \(\bm{y}\) that replaces sequential marginal selection in~\eqref{eq:greedy_marginal_selection} with a single ridge solver on the screened pool \(\mathcal{C}\).\, We solve
\[
\bm{\beta}_{\mathcal{C}}
=
\arg\min\nolimits_{\bm{u}\in\mathbb{R}^{n}}
\|\bm{x}-\bm{D}_{\mathcal{C}}\bm{u}\|_2^2+\lambda\|\bm{u}\|_2^2
\quad\Rightarrow\quad
\bm{\beta}_{\mathcal{C}}
=
(\bm{D}_{\mathcal{C}}^\top\bm{D}_{\mathcal{C}}+\lambda I)^{-1}\bm{D}_{\mathcal{C}}^\top\bm{x}.
\]
Intuitively, this pooled solver estimates coefficients jointly rather than scoring each atom independently.\, When several candidate atoms are highly correlated with each other, raw correlation may assign all of them large scores because they explain a similar direction of \(\bm{x}\).\, The ridge solution uses the Gram matrix \(\bm{D}_{\mathcal{C}}^\top\bm{D}_{\mathcal{C}}\), so redundant atoms compete with each other in the joint fit and can receive smaller coefficients.\, We then keep the Top-\(s\) entries of \(\bm{\beta}_{\mathcal{C}}\) by magnitude and map them back to the full dictionary code \(\bm{y}\in\mathbb{R}^{N}\):
\[
y_j= (\bm{\beta}_{\mathcal{C}})_j, \; j\in \mathrm{TopK}(|\bm{\beta}_{\mathcal{C}}|,s), \qquad y_j=0, \; \text{otherwise}.
\]
Here, \(\bm{\beta}_{\mathcal{C}}\) is indexed by atoms in \(\mathcal{C}\), so \(\mathrm{TopK}(|\bm{\beta}_{\mathcal{C}}|,s)\) returns the \(s\) selected atom indices from \(\mathcal{C}\).\, Thus, the nonzero entries of \(\bm{y}\) are restricted to the screened pool.\, Unlike raw top-$s$ correlation that selects atoms with indices \(\mathrm{TopK}(|\bm{D}^{\top}\bm{x}|,s)\), this pooled solver accounts for correlations among atoms through the Gram matrix $\bm{D}_{\mathcal{C}}^\top\bm{D}_{\mathcal{C}}$.\, Thus, atoms that are individually relevant but redundant with other selected atoms can receive smaller coefficients, giving a one-shot correction toward a more diverse support.\, This variant avoids sequential residual updates and is therefore simpler than greedy marginal, but it can be less faithful to OMP because it performs only one pooled correction.\, We include it as a secondary fast path and study it separately in the ablations.\, Appendix~\ref{app:pool_ridge_details} gives a first-order view of this redundancy correction.

\subsection{Fixed-Size Adaptive Dictionary}

Fast sparse coding reduces encoding cost, but the offline dictionary may still mismatch the test-time decoding distribution.\, FAD-KV monitors the relative error \(e(\bm{x};\bm{D})=\|\bm{x}-\bm{D}\bm{y}\|_2/\|\bm{x}\|_2\), where $\bm{y}$ is the sparse code produced by the current encoder.\, If \(e(\bm{x};\bm{D})>\delta\), adaptation is triggered.\, The threshold $\delta>0$ controls the adaptivity-stability trade-off.\, A smaller $\delta$ triggers more frequent updates and may reduce reconstruction error, but it can also replace atoms more aggressively.\, A larger $\delta$ makes the dictionary more stable but may miss shifted KV states.\, We set $\delta$ on a small calibration split and keep it fixed for all downstream evaluations; the default values are reported in Appendix~\ref{app:experimental_details}.

A direct adaptive strategy appends a new atom from the current vector or residual, but repeated appends make the dictionary grow during generation.\, FAD-KV instead keeps the dictionary size fixed by replacing atoms that are inactive in the current compressed cache.\, Let \(u_j\) be the
number of currently cached sparse codes whose \(j\)-th entry is nonzero:
\(
u_j
=
\sum_{\bm{y}\in\mathcal{Y}_{\mathrm{cache}}}
\mathbf{1}\{y_j\neq 0\},
\mathcal{U}_0=\{j\mid u_j=0\},
\)
where \(\mathcal{Y}_{\mathrm{cache}}\) denotes the set of sparse codes currently
stored in the compressed KV cache.\, Thus, \(\mathcal{U}_0\) contains atom
indices that are not referenced by any cached vector.\, When adaptation is triggered and $\mathcal{U}_{0}$ is nonempty, we construct a new atom from the normalized reconstruction residual and replace one inactive atom:
\[
\bm{a}_{\mathrm{new}}
=
\frac{\bm{x}-\bm{D}\bm{y}}{\|\bm{x}-\bm{D}\bm{y}\|_2+\epsilon},
\qquad
\bm{d}_{j^\star}\leftarrow \bm{a}_{\mathrm{new}},
\qquad
j^\star\in\mathcal{U}_{0}.
\]
Although one could replace the least-used atom $j^\star=\arg\min_j u_j$, replacing an atom with positive current usage would change the reconstruction of cached vectors that reference it.\,
We therefore restrict replacement to $\mathcal{U}_0$, so that no cached sparse code becomes invalid.\,
Since \(u_{j^\star}=0\), no currently stored sparse code uses atom \(\bm{d}_{j^\star}\).\, As a result, previously encoded KV entries remain valid and need not be re-encoded.\,
We then re-encode the current vector over the updated dictionary before storing its sparse code.\,
If $\mathcal{U}_{0}$ is empty, FAD-KV skips the update instead of replacing an active atom.\,
Thus, FAD-KV adapts to test-time activations without global cache updates or dictionary growth.

\section{Theoretical Analysis}
\label{sec:theory}

We analyze the computational and memory cost of FAD-KV and relate KV
reconstruction error to the resulting attention distortion.\,
Full derivations, proofs, and a weak-submodular interpretation of greedy atom
selection are provided in Appendix~\ref{app:theory}.

\subsection{Complexity and Memory Cost}
\label{sec:complexity_memory}

Let \(N\) be the dictionary size, \(n\ll N\) the screened pool size, \(s\) the
sparsity, \(d\) the atom dimension, \(T\) the number of cached KV vectors, and
\(B=b_{\mathrm{idx}}+b_{\mathrm{coef}}\) the per-nonzero storage cost.\,
Full OMP repeatedly scans all \(N\) atoms for \(s\) greedy steps, giving an
atom-search cost of \(O(sNd)\).\, Screened greedy marginal coding instead
requires \(O(Nd+snd)\), while Pool-Ridge removes sequential greedy search and
uses \(O(Nd+n^2d+n^3)\).\, The compressed cache requires
\(O(TsB)\) storage in all cases.\, FAD-KV further keeps dictionary memory fixed
at \(O(Nd)\), whereas append-only adaptation grows to \(O((N+A)d)\) after
\(A\) adaptive updates.\, Full derivations are given in
Appendix~\ref{app:candidate_pruned_gm} and
Appendix~\ref{app:proof_bounded_memory}.

\subsection{Attention Perturbation from KV Reconstruction Error}

\begin{proposition}[Attention perturbation under KV reconstruction]
\label{prop:attn_perturb}
Fix one layer, head, and decoding step $t$, and let
$[t-1]:=\{1,\ldots,t-1\}$.
For $j\in[t-1]$, define
$z_{t,j}=\langle \bm{q}_t,\bm{k}_j\rangle/\sqrt{d}$ and
$\hat z_{t,j}=\langle \bm{q}_t,\hat{\bm{k}}_j\rangle/\sqrt{d}$,
with $\bm{\alpha}_t=\softmax(\bm{z}_t)$ and
$\hat{\bm{\alpha}}_t=\softmax(\hat{\bm{z}}_t)$.
The corresponding attention outputs are
$\bm{o}_t=\sum_{j\in[t-1]}\alpha_{t,j}\bm{v}_j$ and
$\hat{\bm{o}}_t=\sum_{j\in[t-1]}\hat{\alpha}_{t,j}\hat{\bm{v}}_j$.

Assume $\|\bm{q}_t\|_2\le Q$ and
$\|\hat{\bm{v}}_j\|_2\le \hat V$ for all $j\in[t-1]$, and let
\[
E_k=\max_{j\in[t-1]}\|\bm{k}_j-\hat{\bm{k}}_j\|_2,
\qquad
E_v=\max_{j\in[t-1]}\|\bm{v}_j-\hat{\bm{v}}_j\|_2.
\]
Then
\[
\|\hat{\bm{z}}_t-\bm{z}_t\|_\infty
\le QE_k/\sqrt d,
\qquad
\|\hat{\bm{\alpha}}_t-\bm{\alpha}_t\|_1
\le 2QE_k/\sqrt d,
\qquad
\|\hat{\bm{o}}_t-\bm{o}_t\|_2
\le 2\hat VQE_k/\sqrt d+E_v.
\]
\end{proposition}

Proposition~\ref{prop:attn_perturb} shows that key reconstruction error
perturbs attention scores and weights, while value reconstruction error
directly affects the attention output.
This motivates the reconstruction-error analysis in
Section~\ref{sec:efficiency_ablation}; the proof is given in
Appendix~\ref{app:proof_attn_perturb}.





%% file: Sections/Chapter5.tex
\section{Experiments}
\label{sec:experiments}

We evaluate FAD-KV on reasoning and long-context benchmarks, focusing on
whether online adaptation improves robustness to test-time KV shift, whether
fixed-size replacement preserves adaptive benefits without dictionary growth,
and whether screened sparse coding reduces cost while maintaining task quality.
Importantly, append-style adaptive baselines operate under a less constrained
memory budget because they can continuously add atoms during decoding, whereas
FAD-KV keeps the dictionary size fixed.\ We therefore evaluate FAD-KV primarily
as a \textbf{quality--efficiency trade-off} rather than as a method designed to maximize
accuracy alone.

\vspace{-.2cm}

\subsection{Experimental Setup}

We evaluate \textbf{Llama-3.2-3B-Instruct},
\textbf{Llama-3.1-8B-Instruct}~\citep{dubey2024llama}, and
\textbf{Qwen3-4B-Instruct}~\citep{yang2025qwen3}.
For each model, separate key and value dictionaries are trained offline from
WikiText-103 KV activations~\citep{merity2016wikitext} and reused across tasks.
All methods use the same prompts, decoding setup, and evaluation protocol unless
stated otherwise.\ All reported results are \textbf{averaged over three} independent runs.

We compare \textbf{Lexico}, which uses a fixed offline dictionary with full OMP,
and \textbf{Lexico-ADP}, which adds append-style online adaptation.
\textbf{GM} and \textbf{PL} apply greedy-marginal and pooled-ridge coding,
respectively, within the screened pool $\mathcal{C}$ while retaining append-style
dictionary growth.
\textbf{FAD-KV (GM)} and \textbf{FAD-KV (PL)} use the same coding paths but
replace inactive atoms instead of appending new ones, keeping dictionary size
fixed.
Thus, GM/PL isolate the effect of fast sparse coding under adaptive growth,
whereas FAD-KV isolates the additional effect of enforcing a fixed dictionary
budget.
Additional details are provided in Appendix~\ref{app:experimental_details}.
\vspace{-.2cm}

\subsection{Reasoning Benchmarks}

We evaluate reasoning on GSM8K~\citep{cobbe2021training},
MATH500~\citep{hendrycks2021measuring, lightman2023let},
GPQA-D~\citep{rein2023gpqa}, and LongReason~\citep{ling2025longreason}, covering mathematical,
scientific, and long-form reasoning settings.
Table~\ref{tab:reasoning_main} shows that append-style adaptation often
improves over the fixed Lexico dictionary, indicating that online adaptation
can benefit test-time KV reconstruction when dictionary growth is allowed.\ Despite operating under a fixed dictionary budget, FAD-KV recovers part of
this adaptive benefit across multiple model--task settings, although the gain
varies across models and sparse-coding paths.

\begin{table*}[ht]
\centering
\caption{
Accuracy (\%) on reasoning benchmarks.
Best accuracy is bolded within each dictionary group.
}
\label{tab:reasoning_main}

\scriptsize
\setlength{\tabcolsep}{4pt}
\renewcommand{\arraystretch}{0.7}

\resizebox{\textwidth}{!}{%
\begin{tabular}{l|rrrr|rrrr|rrrr}
\toprule

\textbf{Method}
& \multicolumn{4}{c|}{\textbf{Llama-3.2-3B-Instruct}}
& \multicolumn{4}{c|}{\textbf{Llama-3.1-8B-Instruct}}
& \multicolumn{4}{c}{\textbf{Qwen3-4B-Instruct}}
\\

&
\textbf{GSM} & \textbf{MATH} & \textbf{GPQA} & \textbf{LongR.}
&
\textbf{GSM} & \textbf{MATH} & \textbf{GPQA} & \textbf{LongR.}
&
\textbf{GSM} & \textbf{MATH} & \textbf{GPQA} & \textbf{LongR.}
\\

\midrule
\multicolumn{13}{l}{\textit{Adaptive growth}} \\
\addlinespace[1pt]

Lexico-ADP
& 30.0 & 16.4 & 15.7 & \textbf{44.0}
& \textbf{80.0} & 41.4 & \textbf{28.3} & 64.0
& \textbf{65.6} & 31.0 & 21.7 & \textbf{26.0}
\\

GM (ours)
& 26.0 & 17.8 & \textbf{22.2} & 40.0
& 76.0 & \textbf{42.8} & 27.3 & 66.0
& 59.2 & 27.8 & \textbf{22.2} & 25.0
\\

PL (ours)
& \textbf{32.8} & \textbf{19.4} & 17.2 & 38.0
& 74.4 & 41.2 & 27.8 & \textbf{68.0}
& 64.4 & \textbf{33.6} & 21.7 & 20.0
\\

\midrule
\multicolumn{13}{l}{\textit{Fixed-size adaptation}} \\
\addlinespace[1pt]

Lexico
& \textbf{28.4} & 17.0 & 15.2 & 26.0
& \textbf{80.0} & 40.4 & 25.8 & 61.0
& 58.4 & 25.6 & 16.7 & 14.0
\\

FAD-KV (GM, ours)
& 22.0 & \textbf{18.0} & 17.2 & \textbf{32.0}
& 76.0 & \textbf{42.8} & \textbf{27.3} & \textbf{65.0}
& 59.6 & 24.6 & \textbf{20.2} & \textbf{21.0}
\\

FAD-KV (PL, ours)
& 26.8 & 12.8 & \textbf{17.7} & 17.0
& 70.8 & 41.6 & 22.7 & 63.0
& \textbf{61.6} & \textbf{28.4} & 17.7 & 18.0
\\

\bottomrule
\end{tabular}
}

\vspace{-5pt}
\end{table*}

\subsection{RULER}
\label{sec:ruler}

We evaluate long-range retrieval on RULER~\citep{hsieh2024ruler} at 8K and
16K context lengths, covering five representative tasks.\ Table~\ref{tab:ruler_detail} shows that candidate-screened GM and PL closely match
the full-OMP Lexico-ADP baseline across both models, indicating that screening
preserves long-range retrieval while reducing sparse-coding cost.
FAD-KV remains competitive with static Lexico under a fixed dictionary budget,
while the remaining gap to append-style adaptation reflects the trade-off
introduced by fixed-size replacement.

\begin{table*}[t]
\centering
\scriptsize
\renewcommand{\arraystretch}{0.6}

\caption{
RULER results (\%) at 8K and 16K context lengths.
We report NIAH single-key (\textsc{S1}), multikey (\textsc{MK}),
multivalue (\textsc{MV}), multiquery (\textsc{MQ}), and FWE,
together with their average.
}
\label{tab:ruler_detail}

\resizebox{\textwidth}{!}{%
\begin{tabular}{ll|rrrrrr|rrrrrr}
\toprule
\textbf{Model}
& \textbf{Method}
& \multicolumn{6}{c|}{\textbf{8K}}
& \multicolumn{6}{c}{\textbf{16K}}
\\
\cmidrule(lr){3-8}
\cmidrule(lr){9-14}

&&
\textbf{S1}
& \textbf{MK}
& \textbf{MV}
& \textbf{MQ}
& \textbf{FWE}
& \textbf{Avg.}
& \textbf{S1}
& \textbf{MK}
& \textbf{MV}
& \textbf{MQ}
& \textbf{FWE}
& \textbf{Avg.}
\\
\midrule

\multirow{6}{*}{\shortstack[l]{\textbf{Llama-3.2}\\\textbf{3B}}}
& Lexico-ADP
& 65.6 & 43.0 & 19.1 & 17.5 & 44.1 & \textbf{37.9}
& 47.6 & 28.8 & 13.2 & 9.4 & 55.1 & 30.8
\\

& GM (ours)
& 65.6 & 42.8 & 18.3 & 17.1 & 44.1 & 37.6
& 46.8 & 29.2 & 13.5 & 9.3 & 55.6 & \textbf{30.9}
\\

& PL (ours)
& 65.8 & 42.8 & 18.7 & 17.9 & 44.1 & \textbf{37.9}
& 47.4 & 29.4 & 13.4 & 9.5 & 54.3 & 30.8
\\

\cmidrule(lr){2-14}

& Lexico
& 66.0 & 44.8 & 14.5 & 11.9 & 35.3 & 34.5
& 46.6 & 28.6 & 10.9 & 6.8 & 40.7 & 26.7
\\

& FAD-KV (GM)
& 66.2 & 44.8 & 14.4 & 12.8 & 35.7 & \textbf{34.8}
& 48.0 & 29.4 & 11.2 & 7.8 & 40.3 & \textbf{27.3}
\\

& FAD-KV (PL)
& 64.4 & 43.0 & 12.7 & 11.6 & 36.7 & 33.7
& 47.6 & 26.6 & 9.8 & 6.8 & 41.2 & 26.4
\\

\midrule

\multirow{6}{*}{\shortstack[l]{\textbf{Qwen3}\\\textbf{4B}}}
& Lexico-ADP
& 95.5 & 94.0 & 70.0 & 54.0 & 78.5 & 78.4
& 96.0 & 92.0 & 67.2 & 50.5 & 97.0 & 80.5
\\

& GM (ours)
& 95.0 & 93.0 & 71.5 & 57.6 & 77.8 & 79.0
& 96.5 & 90.5 & 69.2 & 50.5 & 96.5 & 80.7
\\

& PL (ours)
& 97.0 & 93.0 & 72.4 & 60.5 & 77.2 & \textbf{80.0}
& 97.0 & 91.0 & 72.2 & 53.1 & 96.8 & \textbf{82.0}
\\

\cmidrule(lr){2-14}

& Lexico
& 94.5 & 95.0 & 57.0 & 29.4 & 78.8 & 70.9
& 95.5 & 92.0 & 54.8 & 29.5 & 96.8 & 73.7
\\

& FAD-KV (GM)
& 95.0 & 94.5 & 57.8 & 30.5 & 79.2 & \textbf{71.4}
& 95.0 & 91.5 & 54.1 & 33.0 & 96.0 & \textbf{73.9}
\\

& FAD-KV (PL)
& 96.5 & 93.5 & 56.1 & 30.2 & 78.8 & 71.0
& 96.0 & 91.5 & 53.4 & 32.2 & 96.2 & \textbf{73.9}
\\

\bottomrule
\end{tabular}%
}

\vspace{-5pt}
\end{table*}

\subsection{LongBench}

\begin{table*}[t]
\vspace{-.4cm}
\centering
\scriptsize
\renewcommand{\arraystretch}{0.7}

\caption{
Accuracy (\%) on LongBench.
Best results are bolded within each dictionary group.
}
\vspace{-.2cm}
\label{tab:longbench_main}

\resizebox{0.9\textwidth}{!}{%
\begin{tabular}{ll|cccccccc}
\toprule
\textbf{Model}
& \textbf{Method}
& \multicolumn{2}{c}{\textbf{Single-Doc}}
& \multicolumn{3}{c}{\textbf{Multi-Doc}}
& \textbf{Summ.}
& \multicolumn{2}{c}{\textbf{Code}}
\\

\cmidrule(lr){3-4}
\cmidrule(lr){5-7}
\cmidrule(lr){8-8}
\cmidrule(lr){9-10}

&&
\textbf{NrtvQA}
& \textbf{MF-en}
& \textbf{Hotpot}
& \textbf{2Wiki}
& \textbf{Musique}
& \textbf{SAMSum}
& \textbf{LCC}
& \textbf{RB-P}
\\

\midrule

\multirow{6}{*}{\shortstack[l]{\textbf{Llama-3.2}\\\textbf{3B}}}
& Lexico-ADP
& 19.50
& 39.94
& 46.99
& 38.52
& 21.09
& 5.60
& \textbf{26.64}
& 27.10
\\

& GM (ours)
& \textbf{19.97}
& 39.94
& \textbf{47.06}
& \textbf{38.63}
& 20.43
& 5.31
& 26.62
& 26.88
\\

& PL (ours)
& 19.20
& \textbf{40.15}
& 46.75
& 38.52
& \textbf{21.10}
& \textbf{5.70}
& 26.60
& \textbf{27.22}
\\

\cmidrule(lr){2-10}

& Lexico
& 16.99
& 38.95
& 46.66
& 37.68
& 20.37
& \textbf{5.70}
& \textbf{26.47}
& \textbf{26.98}
\\

& FAD-KV (GM)
& 18.33
& \textbf{38.40}
& 44.66
& 35.85
& \textbf{19.97}
& 5.33
& 26.45
& 26.75
\\

& FAD-KV (PL)
& \textbf{18.43}
& 38.29
& \textbf{44.86}
& \textbf{36.01}
& 19.46
& 5.33
& 26.43
& 26.66
\\

\midrule

\multirow{6}{*}{\shortstack[l]{\textbf{Llama-3.1}\\\textbf{8B}}}
& Lexico-ADP
& 26.12
& 53.16
& \textbf{55.25}
& 40.65
& 29.72
& \textbf{7.91}
& 26.35
& 25.37
\\

& GM (ours)
& 26.06
& \textbf{54.20}
& 54.31
& 40.65
& 30.81
& 7.84
& 26.28
& \textbf{25.93}
\\

& PL (ours)
& 25.93
& 53.94
& 54.32
& \textbf{41.06}
& \textbf{30.82}
& 7.77
& \textbf{26.35}
& 25.56
\\

\cmidrule(lr){2-10}

& Lexico
& \textbf{26.34}
& \textbf{54.02}
& \textbf{54.75}
& 40.59
& 30.67
& \textbf{7.83}
& 26.21
& 25.89
\\

& FAD-KV (GM)
& 26.09
& 54.00
& 54.32
& 40.65
& \textbf{30.82}
& \textbf{7.83}
& 26.28
& \textbf{25.99}
\\

& FAD-KV (PL)
& 26.02
& 52.88
& 54.12
& \textbf{41.48}
& 30.66
& 7.63
& \textbf{26.36}
& 25.83
\\

\midrule

\multirow{6}{*}{\shortstack[l]{\textbf{Qwen3}\\\textbf{4B}}}
& Lexico-ADP
& 22.67
& 33.42
& 51.40
& 31.31
& \textbf{18.59}
& 17.48
& \textbf{8.82}
& \textbf{11.76}
\\

& GM (ours)
& \textbf{23.32}
& 34.70
& \textbf{51.62}
& 31.62
& 18.14
& 16.94
& 8.21
& 11.36
\\

& PL (ours)
& 23.00
& \textbf{36.53}
& 51.15
& \textbf{34.04}
& 18.29
& \textbf{18.00}
& 7.89
& 9.15
\\

\cmidrule(lr){2-10}

& Lexico
& 6.71
& \textbf{27.70}
& 49.85
& \textbf{31.63}
& 16.61
& 16.33
& 12.49
& \textbf{18.63}
\\

& FAD-KV (GM)
& \textbf{20.86}
& 27.37
& \textbf{50.75}
& 30.95
& \textbf{18.32}
& \textbf{17.15}
& 13.41
& 18.55
\\

& FAD-KV (PL)
& 17.79
& 21.78
& 36.06
& 21.70
& 12.04
& \textbf{17.15}
& \textbf{14.65}
& \textbf{19.00}
\\

\bottomrule
\end{tabular}%
}
\end{table*}

We evaluate long-context performance on LongBench~\citep{bai2023longbench},
covering single-document QA, multi-document QA, summarization, and code
completion.
Table~\ref{tab:longbench_main} shows that append-style adaptation generally
achieves stronger task accuracy, while FAD-KV remains competitive with static
Lexico under a fixed dictionary budget.
This again reflects the central trade-off: append-style updates can improve quality while allowing dictionary growth,
whereas FAD-KV bounds memory and reduces sparse-coding cost.

\subsection{Efficiency and Analysis}
\label{sec:efficiency_ablation}

Table~\ref{tab:efficiency} reports measured decoding efficiency and KV-cache
memory.\ Candidate screening consistently improves decoding throughput over
full-OMP Lexico-ADP.\ In particular, FAD-KV (PL) improves single-sequence
throughput by $12.5\%$ on Llama-3.2-3B and $46.1\%$ on Llama-3.1-8B.\ Although
Full KV remains faster for a single sequence because it avoids sparse
encoding overhead, its per-sequence KV footprint is $4$--$7\times$ larger, which
limits concurrent batch capacity under a fixed memory budget.\ For example, under a 24\,GB memory budget (largest batch that fits,
same GPU), FAD-KV (PL) delivers higher aggregate throughput than Full KV on
Llama-3.2-3B: $1.16\times$ at a 4K context (324 vs.\ 280 tok/s, 9.3 vs.\
21.7\,GB) and $2.0\times$ at 16K (141 vs.\ 71 tok/s).

Figure~\ref{fig:efficiency_ablation} further analyzes the behavior of FAD-KV.
Panel (a) shows that append-style adaptation steadily increases dictionary
size during decoding, whereas FAD-KV remains at its initial budget.
Panel (b) shows that online adaptation reduces reconstruction error relative
to the static Lexico dictionary, and FAD-KV retains this benefit without
dictionary growth.
Panel (c) studies sensitivity to the screening threshold $\tau_c$ and
adaptation threshold $\delta$: performance is relatively stable across a
broad range of screening thresholds, while the adaptation threshold has a
stronger effect on accuracy.
Finally, Panel (d) shows the interaction between dictionary size $N$ and
sparsity $s$; larger representation budgets generally improve accuracy,
with the strongest results obtained at larger $N$ and $s$.

\begin{table*}[t]
\centering
\scriptsize
\vspace{-.2cm}
\renewcommand{\arraystretch}{0.3}
\caption{
Decoding efficiency and KV-cache memory.
}
\vspace{-.2cm}
\label{tab:efficiency}

\resizebox{0.8\textwidth}{!}{%
\begin{tabular}{llccccc}
\toprule
\textbf{Model}
& \textbf{Method}
& \textbf{Cost} $\downarrow$
& \textbf{Latency} $\downarrow$
& \textbf{Decode} $\uparrow$
& \textbf{Speedup} $\uparrow$
& \textbf{Per-seq KV} $\downarrow$ \\
&&&
\textbf{(ms/tok)}
&
\textbf{(tok/s)}
&
\textbf{(compare Lexico)}
&
\textbf{(GB)} \\
\midrule

\multirow{6}{*}{\shortstack[l]{Llama-3.2\\3B}}
& Full KV
& -- & 26.41 & 37.86 & -- & 0.124 \\

& Lexico-ADP
& 1.00 & 53.63 & 18.64 & -- & 0.024 \\

& GM (ours)
& 0.19 & 51.22 & 19.52 & $+4.7\%$ & 0.025 \\

& PL (ours)
& 0.13 & 49.95 & 20.02 & $+7.4\%$ & 0.031 \\

& FAD-KV (GM)
& 0.19 & 52.00 & 19.23 & $+3.1\%$ & 0.020 \\

& \textbf{FAD-KV (PL)}
& \textbf{0.13}
& \textbf{47.66}
& \textbf{20.98}
& \textbf{$+12.5\%$}
& \textbf{0.019} \\

\midrule

\multirow{6}{*}{\shortstack[l]{Llama-3.1\\8B}}
& Full KV
& -- & 30.22 & 33.09 & -- & 0.142 \\

& Lexico-ADP
& 1.00 & 75.62 & 13.22 & -- & 0.040 \\

& GM (ours)
& 0.09 & 69.52 & 14.38 & $+8.8\%$ & 0.040 \\

& PL (ours)
& 0.06 & 56.19 & 17.80 & $+34.6\%$ & 0.045 \\

& FAD-KV (GM)
& 0.09 & 66.61 & 15.01 & $+13.5\%$ & 0.040 \\

& \textbf{FAD-KV (PL)}
& \textbf{0.06}
& \textbf{51.76}
& \textbf{19.32}
& \textbf{$+46.1\%$}
& \textbf{0.040} \\

\bottomrule
\end{tabular}}
\vspace{-.2cm}
\end{table*}

\begin{figure*}[h!]
    \centering
    \includegraphics[width=0.7\textwidth]{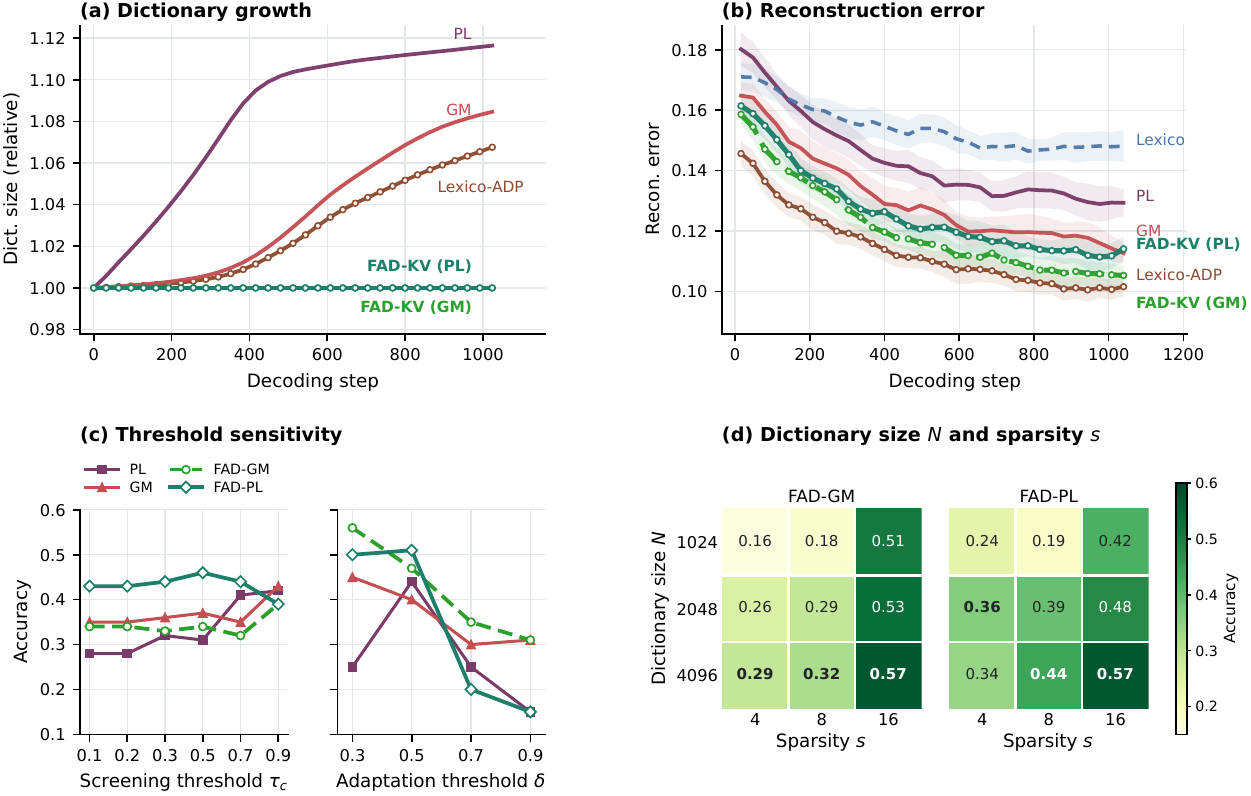}
    \caption{
\textbf{Adaptive behavior and sensitivity analysis.}
(a) Normalized dictionary size during Qwen3-4B-Instruct GSM8K decoding.
(b) Reconstruction error on the same setting, averaged over K/V streams,
layers, and active examples.
(c) Sensitivity of MATH500 accuracy to $\tau_c$ and $\delta$.
(d) Sensitivity to dictionary size $N$ and sparsity $s$.
}
    \label{fig:efficiency_ablation}
    \vspace{-10pt}
\end{figure*}

\vspace{-5pt}

%% file: Sections/Chapter6.tex
\section{Conclusion and Limitations}
\label{sec:conclusion}

We propose FAD-KV, a fast adaptive dictionary KV caching framework for efficient reasoning decoding.\
FAD-KV addresses two practical bottlenecks of dictionary-based KV compression:
iterative sparse-coding latency and append-only dictionary growth.\
By combining candidate-screened sparse coding with fixed-size atom replacement,
FAD-KV adapts to test-time KV states under a strict dictionary budget.\
Our analysis covers sparse-coding cost, dictionary and cache memory,
attention perturbation from KV reconstruction error, and a weak-submodular view
of residual-aware atom selection.\
Future work includes improved replacement policies, adaptive screening, and
evaluation on longer generations and broader model families.

\section*{AI Assistance Disclosure}
Large language models were used to assist with language editing and manuscript
polishing, including improvements to grammar, clarity, and readability. They
were also used to assist with implementing parts of the research code. All
model-assisted code was reviewed, tested, and verified by the authors. The
research ideas, methodology, experimental design, analysis, and scientific
conclusions were developed and validated by the authors, who take full
responsibility for the content of the paper.